%% file: ms.tex
\documentclass{article}
\pdfoutput=1

 \usepackage[preprint, nonatbib]{neurips_2022}

\usepackage[utf8]{inputenc} 
\usepackage[T1]{fontenc}    
\usepackage{hyperref}       
\usepackage{url}            
\usepackage{booktabs}       
\usepackage{amsfonts}       
\usepackage{nicefrac}       
\usepackage{microtype}      
\usepackage{xcolor}         
\usepackage{graphicx}
\usepackage{subcaption}
\usepackage{algpseudocode} 
\usepackage{algorithm}
\usepackage{mathtools}

\usepackage{amssymb}
\usepackage{amsmath}
\usepackage{amsthm}

\newcommand{\abs}[1]{\left| #1 \right|}
\DeclarePairedDelimiter{\norm}{\lVert}{\rVert}

\def\C{\mathcal{C}}
\def\P{\mathcal{P}}
\def\H{\mathcal{H}}
\def\sv{\varphi}
\def\dst{\displaystyle}

\def\polylog{\mathrm{polylog}}

\newtheorem{theorem}{Theorem}[section]
\newtheorem{lemma}[theorem]{Lemma}
\newtheorem{observation}[theorem]{Observation}
\newtheorem{definition}[theorem]{Definition}

\newcommand{\myparagraph}[1]{\noindent {\bf #1}}

\title{Differentially Private Shapley Values for Data Evaluation}

\author{%
  Lauren Watson \\
  School of Informatics\\
  University of Edinburgh\\
  \texttt{lauren.watson@ed.ac.uk} \\
   \And
   Rayna Andreeva \\
   School of Informatics \\
   University of Edinburgh \\
   \texttt{r.andreeva@sms.ed.ac.uk} \\
   \AND
   Hao-Tsung Yang \\
   School of Informatics \\
   University of Edinburgh \\
   \texttt{haotsungyang@gmail.com} \\
   \And
   Rik Sarkar \\
   School of Informatics \\
   University of Edinburgh \\
   \texttt{rsarkar@inf.ed.ac.uk} \\
}

\begin{document}

\maketitle

\begin{abstract}
    The Shapley value has been proposed as a solution to many applications in machine learning, including for equitable valuation of data. Shapley values are computationally expensive and involve the entire dataset. The query for a point's Shapley value can also compromise the statistical privacy of other data points. We observe that in machine learning problems such as empirical risk minimization, and in many learning algorithms (such as those with uniform stability), a diminishing returns property holds, where marginal benefit per data point decreases rapidly with data sample size. Based on this property, we propose a new stratified approximation method called the Layered Shapley Algorithm. We prove that this method operates on small ($O(\polylog(n))$) random samples of data and small sized ($O(\log n)$) coalitions to achieve the results with guaranteed probabilistic accuracy, and can be modified to incorporate differential privacy. Experimental results show that the algorithm correctly identifies high-value data points that improve validation accuracy, and that the differentially private evaluations preserve approximate ranking of data. 
\end{abstract}

\section{Introduction}

Large-scale machine learning and data mining depend on data contributed by various individuals and institutions. With the popularity of such data-driven systems, there is increasing awareness of the value of data and associated privacy risks. As data sharing and data marketplaces become common, it is necessary to accurately evaluate a contributor's data to provide them with the right compensation. On the other hand, identifying high-value data is also of advantage to other stakeholders such as data users. For these purposes, the Shapley value has been proposed as a fair method of determining the value of each data point~\cite{jia2019efficient,liu2021gtg}. 

The Shapley value is a concept from game theory~\cite{shapley1953value}, defined to evaluate the contribution of individual players in a cooperative game. 
The concept is very general and can be applied to complex setups where multiple elements interact to produce results. Thus, it has been applied to understanding different elements in machine learning and is a popular tool in interpretable machine learning~\cite{lundberg2017unified}. It is used to determine the importance of individual features~\cite{cohen2007feature,tripathi2020interpretable}, neurons~\cite{ghorbani2020neuron}, models in ensembles~\cite{rozemberczki2021shapley}, data points~\cite{ghorbani2019data, jia2019efficient} and many others (see~\cite{rozemberczki2022shapley}). However, the use of Shapley values is challenging from the perspectives of computation and privacy. 

The computation of Shapley values requires evaluating the marginal contribution of a player (for us, a data point) with respect to all possible coalitions (subsets) of players (See Section~\ref{sec:prelim}). The computational problem is $\#P$-complete~\cite{deng1994complexity}. Additionally, each such evaluation involves training and evaluating a model. Approximations based on Monte Carlo sampling~\cite{maleki2013bounding} of coalitions can reduce the cost to $O(n\log n)$ evaluations, which is still prohibitive in large datasets. The reduction of computation (number of evaluations) has been the focus of several recent works on the Shapley value of data~\cite{ghorbani2019data,jia2019efficient,kwon2021efficient}. 

From the privacy perspective, the Shapley value poses a complex challenge, since the contribution of a data point is influenced by the configuration of all other data in the set. Our objective is to answer a query for $\sv_i$: Shapley value of data point $i$, which will require access to the entire dataset. Even Monte Carlo methods that sample random coalitions require the use of almost all data points (Section~\ref{sec:algorithm}). Conversely, answering a query for $\sv_i$ can leak the privacy of any other data point $j$. To address this challenge, we develop an algorithm specifically for data evaluation that heavily samples smaller coalitions. This algorithm is more compatible with differential privacy and needs to access only a small fraction of data. Thus it is more privacy-friendly than existing methods. 

\myparagraph{Our contributions.} Our approach is based on diminishing marginal gains with increasing data volume. In machine learning problems, larger training datasets are desirable, but the incremental benefit per data point decreases with increasing data size. This effect has been seen in past experiments (e.g.~\cite{Sun_2017_ICCV} -- Fig. 4). We show theoretically, that for an empirical risk minimization (ERM) problem, the marginal reduction in loss per data point is inversely proportional to the data size i.e. $O(1/n)$ (Subsection~\ref{sec:diminishing-return}). Thus, each new data point contributes less to the objective of loss minimization in larger datasets.  Similar theoretical bounds on marginal differences hold for uniformly stable algorithms such as regularised ERM. 

Using the property that the marginal utility of a data point is bounded by $O(\frac{1}{m})$ for coalitions of size $m$, we devise a Shapley value computation algorithm in Subsection~\ref{sec:layered-shapley}. This algorithm stratifies the coalitions into layers and samples the layers with varying probability. We call this the {\em Layered Shapley Algorithm}. The algorithm heavily samples the lower layers with small coalitions, and sparsely samples the higher layers with large coalitions. The intuition is that, given the diminishing returns property, small coalitions provide sufficient information on a data point's utility. Not much remains to be gained by examining large coalitions, where the marginal utility is anyway guaranteed to be small. This algorithm relies on evaluating $O(\ln n)$ coalition samples and uses only a $O(\frac{\ln^2 n}{n})$ fraction of the dataset. It is thus highly efficient in the number of evaluations and data usage. 

In Subsection~\ref{sec:dp-shapley} we discuss the differentially private Shapley value computation based on the Layered Shapley Algorithm. We use the strong sampling property of the algorithm with sampling-based privacy amplification results to get differential privacy at the price of relatively small noise. 

Some properties of the Layered algorithm and related results are discussed in Subsection~\ref{sec:layered-properties}. We observe that the bias towards evaluating smaller coalitions significantly helps the computational costs since small coalitions are cheaper to train. The $O(\ln^2 n)$ data points and $O(\ln n)$ coalitions needed to compute a value $\sv_i$, can be saved and used to answer Shapley value queries on the same dataset. Thus, the Layered algorithm produces a natural small {\em core set} for querying Shapley values. Finally, we argue that in the realistic case where a contributor may submit a set of points, the aggregate evaluation of the set can be carried out at a small relative cost. 

Experimental results are discussed in Section~\ref{sec:experiments}, where we demonstrate that Shapley values calculated via the Layered Shapley Algorithm, and their differentially private counterparts, successfully describe the relative utility of data points within multiple binary classification tasks despite their reliance on small coalition sizes. The private algorithm approximately preserves the relative ranks of data points compared to the non-private version. Related works are discussed in Section~\ref{sec:related}. Proofs of theorems can be found in the appendix. In the next section, we start by reviewing the background on Shapley value and differential privacy. Readers familiar with the topics may want to quickly skim the section to note the definitions and notations.

\section{Preliminaries}
\label{sec:prelim}

\subsection{The Shapley value}

The Shapley Value was originally developed to evaluate the contributions of different players in a cooperative game~\cite{shapley1953value}.  In machine learning, the set of players are usually elements of the input to the training algorithm. For example, the input features can be treated as players to estimate their relative importance. Analogously, to evaluate the relative importance of different data points, they will be treated as players and the Shapley value of a data point will represent its importance in the training process. 

In a game with $n$ players (which may be $n$ features or $n$ data points as the case may be), the Shapley value of player $i$, written as  $\sv_i$, is defined in terms of their marginal contributions to coalitions of other players. Suppose $N$ is the set of $n$ players, and  $\C=2^N$ is the set of all possible subsets (coalitions) of players. The utility obtained by any coalition $C$ is given by a value function $v$, and the marginal contribution of player $i$ with respect to $c$ is written as $v_i (C) = v(C\cup\{i\}) - v(C)$. The Shapley value is then defined by: 
\begin{equation}\label{eq:shapley}
   \sv_i(v) = \frac{1}{n}\dst\sum_{C\subseteq N\setminus\{i\}} {n-1 \choose \abs{C}}^{-1}\cdot v_i(C).
\end{equation}
Observe that this definition is equivalent to computing the average marginal gains of $i$ over coalitions of each possible size, and then averaging over all possible sizes.

The appeal of the Shapley value is that it provides a fair allocation of credit, more meaningful than simple marginal contributions. This fairness is characterized by several intuitive properties, such as efficiency, symmetry, null player, and linearity. Shapley value is the unique valuation function that satisfies all these. See the  survey~\cite{rozemberczki2022shapley} for details of these properties. 

In the context of machine learning, $v$ is often defined in terms of the loss function, measuring how much an element $i$ helps in minimizing loss. In the typical data evaluation problems~\cite{ghorbani2019data,ghorbani2020distributional,jia2019towards}, each data point is treated as a player, and $h$ is the model trained on coalition $C$. If $L_C(h)$ is the loss of $h$ on $C$, then $v(C)$ can be defined as $v(C)=-L_C(h)$.  Thus, the Shapley value $\sv_i$ is larger for data points that help more in minimizing the loss. Both the training or empirical loss~\cite{tripathi2020interpretable} and validation loss~\cite{jia2019towards} have been used to define $v$ in machine learning research.

In the feature evaluation problem~\cite{guyon2003introduction,fryer2021shapley}, each feature is treated as a player, and for a subset $C$ of features, $v(C)$ is defined analogously in terms of the loss.

\subsubsection{Approximate computation of Shapley value} 
\label{sec:basic-shapley-approx}
The definition of the Shapley value requires computing $v_i(C)$ for an exponential number of coalitions, making it computationally expensive. The typical approach to tractable computation is to perform a Monte Carlo estimate over the set of coalitions. Suppose $\pi$ is a permutation of $N$, taken uniformly at random, and $\P^\pi_i$ is the set of items occuring before $i$ in $\pi$. Then the basic sampling based algorithm~\cite{castro2009polynomial} computes the average marginal gain over a sample of such subsets to obtain the approximate Shapley value: $\hat{\sv}_i =\frac{1}{m} \sum_{j=1}^m v_i(\P^\pi_i)$. In the case of all $v_i(C)$ being bounded by a constant $c$, the sample complexity of $m\geq \left\lceil\dst \frac{\ln{(\frac{2}{\beta})c^2}}{2\alpha^2} \right\rceil$ achieves an $(\alpha,\beta)$-approximation guarantee~\cite{maleki2013bounding}: $\Pr(\abs{\hat\sv_i - \sv_i}\geq \alpha)\leq \beta$. 

\subsection{Differential Privacy}
The privacy of data points $z\in D$ is at risk even when computing a seemingly complex aggregate value such as a machine learning model~\cite{shokri2017membership}, or in our case a Shapley value. The computation of Shapley value $\sv_i$ uses every other data value $j$ and thus risks their privacy. Differential privacy~\cite{Dwork06differentialprivacy} is designed to defend against such privacy leaks. It provides a statistical privacy guarantee for all data points $z\in D$ by ensuring that the value is statistically insensitive to the presence or absence of individual data points.
\begin{definition}[\bf Neighbouring Databases]
Two databases $D, D'$ are neighbouring if $H(D, D') \leq 1$, where $H(\cdot, \cdot)$ represents the hamming distance. 
\end{definition}
\begin{definition}[\bf Differential Privacy~\cite{Dwork06differentialprivacy}]
  A randomized algorithm $M$ satisfies $\epsilon$-differential privacy if for all neighbouring databases $D$ and $D'$ and for all possible outputs $O\subseteq \text{Range}(M)$, \mbox{
  $
      \Pr[M(D) \in O] \leq e^{\epsilon} \cdot \Pr[M(D') \in O].
  $}
\end{definition}

The sensitivity of a function $f$ is defined to be the maximum change in the function value between neighboring databases: $\Delta f = \max_{D, D' \in \mathcal{D}} \lvert f(D)-f(D') \rvert$. The sensitivity determines the appropriate scale of noise to add to $f$ to achieve differential privacy, as follows: 

\begin{theorem}[Laplace Mechanism] Given a function $f:(\mathcal{X}\times\mathcal{Y})^n\rightarrow \mathbb{R}^k$, the Laplace Mechanism releasing $f(D)+r, r\stackrel{k}{\sim}  Lap(0,\frac{\Delta f}{\epsilon})$ satisfies $\epsilon$-differential privacy. 
\end{theorem}

As we will discuss, one approach to releasing a privacy-preserving Shapley value is to determine its sensitivity and add the appropriate amount of noise. The challenge will be to do this while maintaining accurate estimates of the value.

\begin{table}[h]
    \centering
    \begin{tabular}{c|c|c|c}
     Symbol    & Definition & Symbol & Definition  \\ \hline
       $\sv_i$   & Shapley value of data point $i$ &
       $n$ & Data size \\
       $D, N$   & Data set &
       $S$   & Set of permutation samples \\
       $m$   & Sample size &
       $C$   & A coalition \\
       $\C$   & All coalitions &
       $\C_k$   & All coalitions of size $k$\\
       $v$   & Valuation function &
       $v_i(C)$   & Mariginal gain of $i$ over $C$ \\
       $\alpha,\beta$   & Approximation parameters &
       $\epsilon$ & Differential privacy
    \end{tabular}
    \vspace*{2mm} 
    \caption{Frequently used notations.}
    \label{tab:Notation}
\end{table}

\section{Algorithms and analysis}
\label{sec:algorithm}
We have discussed that the computation of Shapley values is expensive. Even with sampling-based approximations, large fractions of the dataset are used to answer a single query for $\sv_i$. To see this, consider a single random permutation $\pi$. With probability at least $1/2$, data point $i$ is in the second half of $\pi$. Thus with a probability of at least $1/2$, $\abs{\P^\pi_i}\geq n/2$ and at least half the dataset will be required to compute a single marginal value. Since the computation of $\sv_i$ involves many such marginal valuations, nearly the entire dataset is used to answer a query for a Shapley value. In addition to the risk of exposing all data to the agent performing the computation, the large coalition sizes create challenges in terms of differential privacy, since a query for any $\sv_i$ may reveal information about any other data point. 

In this section, we argue that by using the specific properties of the marginal loss in machine learning, we can improve upon these issues -- with an estimation algorithm that uses only a small fraction of data. In the following subsection we analyze the intuitive idea that in larger datasets, the marginal contributions of individual data points are proportionally smaller.

\subsection{Diminishing marginal gains with data} \label{sec:diminishing-return}

In this subsection, we discuss how increasing data volume reduces the marginal gain per data point (e.g.~\cite{Sun_2017_ICCV} -- Fig. 4). With increasing data, the algorithm approaches the optimal model, and the loss converges to the minimum, with tinier steps. This effect can be seen more formally in the case of empirical risk minimization using the simple setup of binary classification with $0-1$ loss~\cite{shalev2014understanding}. Given a set $C$ of size $m$  with labelled data points $(x_i, y_i)$, the empirical risk of a model $h$ is defined by $L_C(h)=\frac{1}{m} \abs{\{ i\in [m]:h(x_i)\neq y_i \}}$ -- that is, the fraction of points incorrectly classified by $h$. The models $h$ are drawn from a hypothesis class $\H$. The optimal model $h^\star_C$ in the class is the one that minimizes the loss over $C$. Since the loss is an average over the number of data points, the introduction of an additional data point $x$ can only change the risk by $O(\frac{1}{m})$: 
\begin{observation}
For any subset $C$ and new data point $x$, the marginal change in loss of the optimal model is bounded by $\abs{L_C(h^\star_C) - L_{C\cup\{x\}}(h^\star_{C\cup\{x\}})}\leq \frac{1}{m}$. 
\end{observation}

In machine learning, a natural value function $v$ is defined by the empirical (training) loss: $v(C) = - L_C(h^\star_C)$. Or, if an upper bound $L_{max}$ on $L$ is known, then possibly $v(C) = L_{max} - L_C(h^\star_C)$. In either case, for a data point $i$ and any set $C$, the observation above implies a bound on the marginal gain of $i$ w.r.t $C$: $v_i(C)\leq \frac{1}{m}$. 

\myparagraph{Regularized ERM and stability.} The $O(1/m)$ bound on marginal difference holds more generally in stable machine learning. One of the commonly used stability notions is Uniform Stability~\cite{bousquet2002stability}. Suppose the dataset $D$ of size $n$ contains points $z_i =(x_i, y_i)$ for $i\in\{1, ..., n\}$ from the domain $\mathcal{Z}=\mathcal{X} \times \mathcal{Y}$. Let  $D^{\backslash  i}$ represent $D$ with data point $i$ removed. Suppose we write $\ell(A_D, z)$ to denote the loss on a point $z$ of the model computed by algorithm $A$ on data $D$. Given this, $\gamma$-Uniform Stability ensures that the change in the loss for any datapoint $z\in\mathcal{Z}$  is bounded by $\gamma$ when any individual point is removed from the training set:
\begin{definition}[Uniform stability~\cite{bousquet2002stability}] A learning algorithm $A$ has $\gamma$-uniform stability with respect to the loss function $\ell$ if the following holds,
\[\forall z\in\mathcal Z, \forall D \in \mathcal{Z}^m, \forall i \in \{1, ..., n\}, \norm{\ell(A_{D}, z)-\ell(A_{D^{\backslash  i}}, z)}_{\infty}\leq \gamma.\]
\end{definition}

The definition implies a  $O(\frac{1}{m})$ bound on the marginal gain of $v$. For example, $\gamma=\frac{L^2 \kappa^2}{2\lambda m}$ for regularized algorithms such as L2-regularized regression in reproducing kernel Hilbert spaces with kernel $k(x, x)\leq \kappa^2$, regularization strength $\lambda$ and Lipschitz constant $L$~\cite{bousquet2002stability}. When $v$ is defined to be either the averaged empirical or validation loss, this implies that the marginal difference is bounded by $\gamma$, and so is $O(\frac{1}{m})$.

Uniformly stable algorithms are known to have strong generalization properties~\cite{bousquet2002stability,feldman2018generalization} and for this reason, are commonly used in research and practice. For example, regularized ERM methods such as linear and logistic regression with L2 regularization satisfy this property. Several other forms of regularizers and learning algorithms satisfy the property as well (See~\cite{bousquet2002stability,audiffren2013stability}). For popular techniques such as stochastic gradient descent, there has been recent progress in establishing stability. Uniform stability with $O(\frac{1}{m})$ marginal differences is known to hold for SGD in expectation even in non-convex cases~\cite{hardt2016train}. 

Next, we see how the property of $O(\frac{1}{m})$ marginal differences can be used to design improved algorithms for the Shapley value of data. 

\subsection{Layered Shapley value Algorithm}
\label{sec:layered-shapley}

\input{temp_stratified_sampling}

\subsection{Differentially Private Shapley Values}
\label{sec:dp-shapley}
\input{privacy}

\subsection{Properties and other observations} 
\label{sec:layered-properties}

\myparagraph{Computation and data access costs.} Compared to algorithms that use a large number of coalition samples and almost all data points, the Layered Shapley approach works with $O(\polylog(n))$ data points. This is a system advantage, since accessing large datasets can incur many disk/ network/ device access costs. 

Computationally, the small data requirement implies that the average coalition is only $O(\polylog(n))$ in size. Since a training algorithm needs to run for each coalition, this gives a large advantage. For example, assuming that the training algorithm in question runs in $\approx \mathrm{poly}(n)$ time, one of the traditional approximation algorithms that require $\Omega(n)$ data points, will require a $\Omega(\mathrm{poly}(n))$ running cost. Whereas, the Layered algorithm will run in $O(\polylog(n))$ time. 

\myparagraph{Small sets for evaluations.} The $S_k$ sets generated on a run of the algorithm can be saved and treated like a core set -- a small sample of a large dataset that serves to approximate results for future queries. In such a setup, each contributor needs to submit only a small fraction of their data for the general service of data evaluation. This is more privacy-friendly and likely to be acceptable for both individuals and institutions.

\myparagraph{Valuation of data subsets.} In practice, it is likely that a single contributor submits multiple data points, and the point of interest is that the total or average value (and corresponding compensation) is accurate. 

If a person submits $w$ data points, then it follows from a simple probabilistic analysis, that to ensure that the average cost is within an error of $\alpha$, a sample complexity if $O(\ln(nw))$ suffices. Thus, multiple data contributions and queries for subsets effectively decrease the samples and costs.

\input{experiments}

\input{related_work}

\section{Conclusion}

We address the privacy issue of Shapley value-based data valuation and propose the Layered Shapley value algorithm. The algorithm preserves differential privacy and utilizes the diminishing marginal gain to provide efficient computation. The theoretical bound does not extend to algorithms without uniform stability, such as training neural networks. Stability results in~\cite{hardt2016train} that hold for SGD in expectation suggest that suitable results may be derived in the future. In experiments, we find that both differentially private and non-private Shapley values computed by our algorithm are still useful compared with the baseline. 

We imagine a system where individuals can easily obtain valuations of their data. The theoretical results in this paper provide algorithms, but we are still far from widely usable systems. A major challenge in such a system will be to obtain meaningful value (e.g. compensation) instead of abstract numbers, which will be hard to translate to social value. In such real systems, Shapley value may or may not be the right approach. Its axiomatic properties are often cited as the reason to use it, but to what extent these hold in the Monte Carlo approximations remain to be investigated. It is also unclear if, in the case of data contributions, Shapley values agree with the human intuition of value. A few works have suggested that people may overlook the properties of the Shapley value and have incorrect expectations of it~\cite{kumar2020problems,fryer2021shapley}.  

\myparagraph{Social impact.} While this work contributes to the domain of ethical machine learning research by extending data valuation techniques to include privacy-preserving valuation, its social impacts can include negative elements. Both Shapley value and differential privacy are non-trivial concepts, and it is unclear if a system combining the two actually helps people, in general, make better decisions, or will simply add to greater uncertainty and fear of technology. Differential privacy, for example, does not provide absolute privacy, but rather a probabilistic one and is dependent on $\epsilon$, which may be a source of misunderstanding. Differential privacy is also known to have disparate impact~\cite{bagdasaryan2019differential,ganev2021robin}, and it is unclear if in this case, the algorithm will maintain the fairness of valuations. 

The use of such data valuation services themselves may be susceptible to attacks, frauds, and abuse. Unethical players may contribute spurious data points with the objective of increasing their own values or disrupting that of others. Valuation systems may perpetuate fraud and introduce the issue of monitoring the agent performing valuations. Leaked data valuations may make high-value data holders subject to attacks and fraud. The idea of incentivizing the contribution of (high value) data, while useful in theory, comes with some potential for abuse. Depending on the circumstances, it can be seen as a coercion to contribute data or a penalty for not contributing data. Specifically, high-value data is also likely to be privacy sensitive, and the incentives can be seen as a push toward loss of privacy.

\bibliography{ref}
\bibliographystyle{abbrv}
\input{appendix}

\end{document}

%% file: temp_stratified_sampling.tex
Our approach to designing an efficient algorithm is to leverage the assumption that the mariginal difference in the value of a coalition $C$ on addition of any single data point $i$ is bounded by $\abs{v_i(C)}\leq \frac{c}{k}$, where $k$ is the size of $C$ and $c$ is a constant independent of $C$. 

With this assumption, we can increase the probability of small coalitions being evaluated, since the difference in value increases slowly with coalition size. The algorithm is presented as Algorithm~\ref{algo:stratified-sampling}. It operates by stratifying the coalitions into layers by their sizes, and then estimating the expected marginal gain from $i$ in each layer. The algorithm is analogous to selecting $m_k$ random coalitions from layer $k$. Observe that $m_k$ drops rapidly as the coalition sizes increase. Where $m_k$ is smaller than $1$, the algorithm is probabilistically equivalent to sampling layer $k$ with probability $m_k$. 

\begin{algorithm}[hbt]
\caption{Layered Shapley Algorithm}\label{alg:cap}
    \begin{algorithmic}[1]
    \State Input: $(\alpha,\beta)$: approximation parameters, $n$: number of points, $N$: set of points, $v$: loss function, $c$: constant in bound for marginal change
    \State Output: $\hat{\varphi}_i$: the estimated Shapley value of datum $i$
    \For{$k$ from $1$ to $n-1$} \Comment{For each layer}
        \State $m_k \leftarrow \frac{c^2}{2\alpha^2 k^2} \ln{\frac{2n}{\beta}}$
        \State $w_k \leftarrow {n-1 \choose k}$
        \State $p_k \leftarrow m_k/w_k$ \Comment{Probability of a coalition in layer $k$ being used}
        \State Draw $S_k$ where $\forall C\in\C_k, \Pr(C\in S_k) = p_k$ \Comment{Draw a sample of coalitions from layer $k$}
        \State $\hat{\phi}_i^k\leftarrow \frac{1}{p_k}\frac{1}{w_k}\sum_{C\in S_k} v_i(C)$ \Comment{Estimate of average marginal gain in Layer $k$ }
    \EndFor
    \State return $\hat{\varphi}_i = \frac{1}{n}\sum_{k=1}^{n-1} \hat{\phi}_i^k$
    \end{algorithmic}
\label{algo:stratified-sampling}
\end{algorithm}

\begin{theorem}
\label{thm:ab-log_n}
The estimate $\hat{\varphi}_i$ is an $(\alpha, \beta)$ approximation, that is, $\Pr(\abs{\hat{\sv}_i - \sv_i}\geq \alpha) \leq \beta$, and is computed using a coalition sample complexity of $O(\ln n)$. 
\end{theorem}

The proof of Theorem~\ref{thm:ab-log_n} is in the appendix. The proof essentially relies on Hoeffding's bound to show that the estimate $\hat\phi_i^k$ for each layer is within a small error, and uses the union bound to argue that the average error over all layers is probably small. 

We have noted earlier that access to large data volumes for each query is undesirable. The following theorem shows that on each query, the algorithm only needs to access a small fraction of data: 

\begin{theorem}
The probability that data point $j$ is used in the computation of $\varphi_i$ is bounded by $\frac{c^2\ln n}{2\alpha^2 n}\ln\frac{2n}{\beta}$. 
\end{theorem}\label{thm:sampling-prob}

In the next section we will see how to use this result to provide differentially private Shapley values.

%% file: privacy.tex
Using the layered sampling approach outlined in Algorithm~\ref{algo:stratified-sampling} together with the bounded marginal contributions discussed in Section~\ref{sec:diminishing-return}, a differentially private Shapley value can be released via the Laplace mechanism~\cite{dwork2006calibrating}\footnote{Note that this approach could be trivially extended to satisfy $(\epsilon, \delta)$-differential privacy via the Gaussian mechanism}. For problems with marginal contributions bounded by $O(1/k)$, the sensitivity of the Shapley value is also bounded and can be used to ensure differential privacy. In this case, this output perturbation approach is preferable to perturbing intermediate steps of the algorithm (e.g. using private machine learning to evaluate $v$). This is due to both the necessity of composition over $2m$ evaluations (where $m$ is the total number of coalitions evaluated) in that case, as it involves evaluating 2 machine learning models per sampled coalition, and the fact that private machine learning is designed in principle to mask the differences due to a single point that are measured by the marginal contribution. Instead, we combine layered sampling with the bounded sensitivity to release the private Shapley value (Algorithm~\ref{alg:cap}).
\begin{algorithm}[hbt]
\caption{Private Layered Shapley Algorithm}\label{alg:cap}
    \begin{algorithmic}[1]
    \State Input: $(\alpha,\beta)$: approximation parameters, $n$: number of points, $N$: set of points, $v$: loss function, $\sigma$: noise scale.
    \State Output: $\hat{\varphi}_i^{priv}$: the private estimated Shapley value of datum $i$
    \State $\hat{\varphi}_i =$ Layered-Shapley($(\alpha,\beta),N,v$) \Comment{Output of Algorithm~\ref{algo:stratified-sampling}}
    \State $\hat{\varphi}_i^{priv} = \hat{\varphi}_i +r$, $r\sim Lap(0,\sigma)$
    \State return $\hat{\varphi}_i^{priv}$
    \end{algorithmic}
  \label{alg:private-shapley}
\end{algorithm}

This algorithm can be shown to be $\epsilon$-differentially private (Thm.\ref{thm:privacy}). 
\begin{theorem}\label{thm:privacy}
Algorithm~\ref{alg:private-shapley} satisfies $\epsilon$ -differential privacy with noise scale $\sigma=\frac{L^2\kappa^2}{m\lambda \ln(\frac{e^{\epsilon}-1}{p}+1)}\sum_{i=1}^n \frac{m_k}{k}$ where $p=\frac{c^2\ln n}{2\alpha^2 n}\ln\frac{n}{\beta}$. 
\end{theorem}
The proof first demonstrates that the sensitivity of the approximated Shapley value is given by $\frac{L^2\kappa^2}{m\lambda}\sum_{i=1}^n \frac{m_k}{k}$ and then makes use of the fact that any data point has a small probability $p$ of being used. This allows us to use results of privacy amplification by sampling~\cite{ kellaris2013practical, beimel2014bounds, balle2018privacy} to obtain differential privacy without excessive noise.

%% file: experiments.tex
\section{Experiments}\label{sec:experiments}
We now provide empirical results demonstrating the efficacy of our algorithms on binary classification tasks with regularized logistic regression.

\textbf{Experimental Setup:} Experiments were performed using both publicly available binary classification datasets and synthetic data matching the dataset used by~\cite{ghorbani2019data}.\footnote{The code used in these experiments is an extension of \url{https://github.com/amiratag/DataShapley}~\cite{ghorbani2019data} which is licenced under the MIT License(See \url{https://github.com/amiratag/DataShapley/blob/master/README.md}). The Adult dataset is licensed under a Creative Commons Attribution 4.0 International (CC BY 4.0) license and the Diabetes dataset under the Creative Commons Attribution 1.0 Universal (CC0 1.0) license.} The publicly available datasets used were the Adult dataset~\cite{kohavi1996scaling} and the Diabetes dataset from the UCI Machine Learning Repository~\cite{frank2010uci}. The synthetic data was generated by following the synthetic data generation approach of~\cite{ghorbani2019data} including sampling features from a 50-dimensional multidimensional Gaussian distribution $\mathcal{N}(0, I)$. All experiments use the Scikit-Learn~\cite{pedregosa2011scikit} implementation of regularized logistic regression and the appropriate noise scale for $\epsilon=1$. Private algorithm performance was reported as an average over 5 runs. In these experiments $v$ is defined to be the negative heldout loss and coalitions with $v$ below the random guessing baseline were not included. See the Appendix for further experimental details.

\begin{figure}[htp]
     \centering
     \begin{subfigure}[b]{0.3\textwidth}
         \centering
         \includegraphics[width=\textwidth]{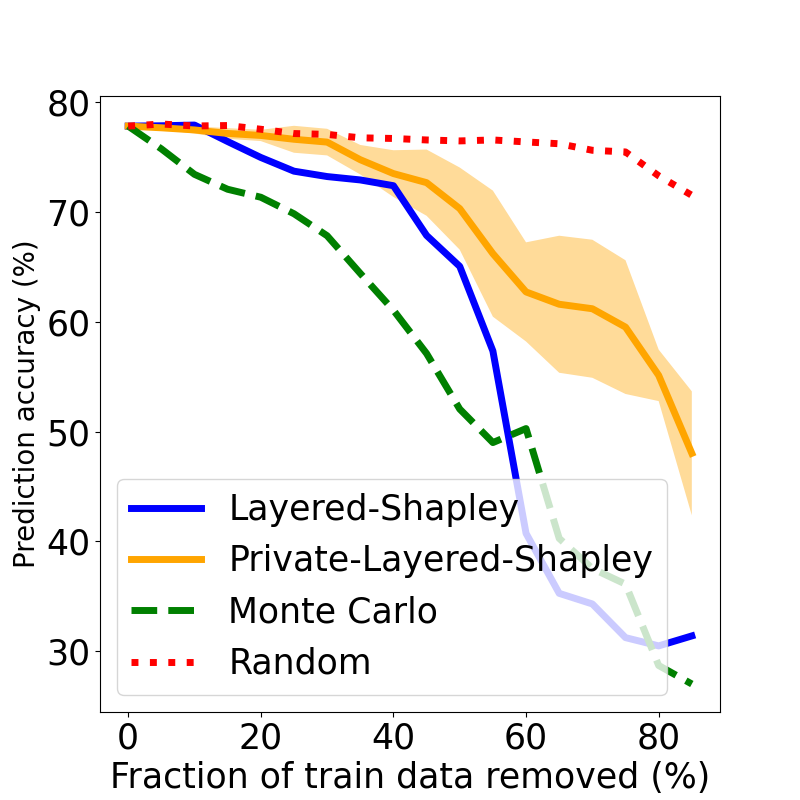}
         \caption{Adult ($\rho=0.61$)}
         \label{fig:y equals x}
     \end{subfigure}
     \hfill
     \begin{subfigure}[b]{0.3\textwidth}
         \centering
         \includegraphics[width=\textwidth]{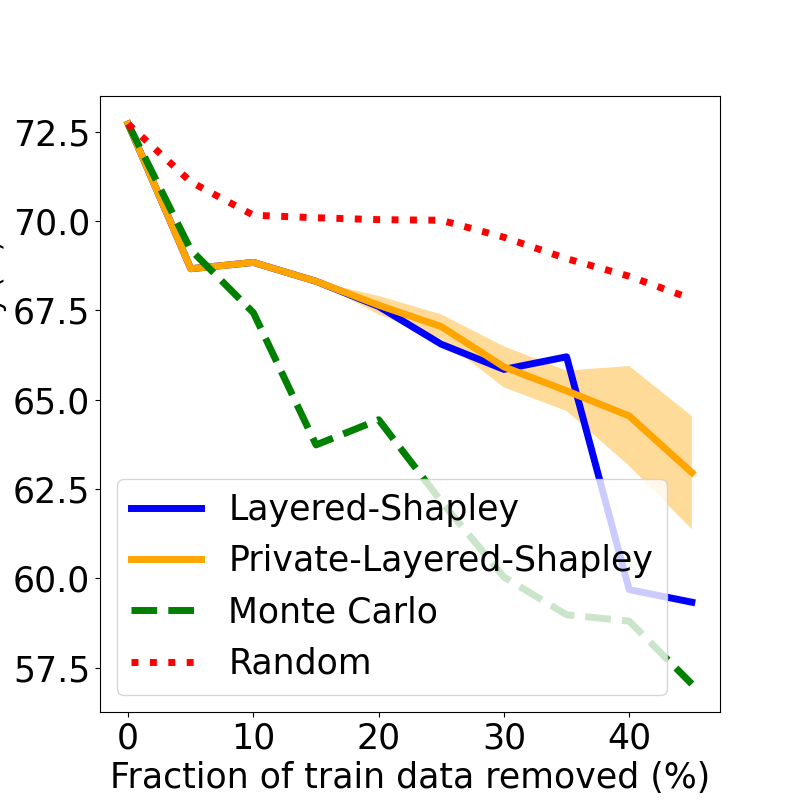}
         \caption{Diabetes ($\rho=0.89$)}
         \label{fig:three sin x}
     \end{subfigure}
     \hfill
     \begin{subfigure}[b]{0.3\textwidth}
         \centering
         \includegraphics[width=0.98\textwidth]{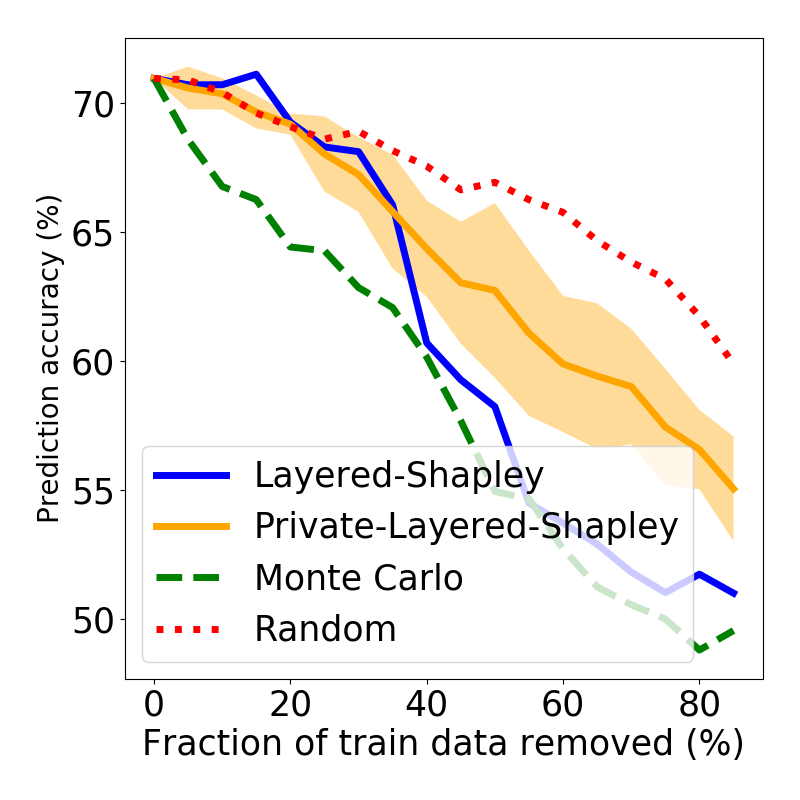}
         \vspace{-3mm}
         \caption{Synthetic Dataset ($\rho=0.52$)}
         
         \label{fig:five over x}
     \end{subfigure}
     \begin{subfigure}[b]{0.3\textwidth}
         \centering
         \includegraphics[width=\textwidth]{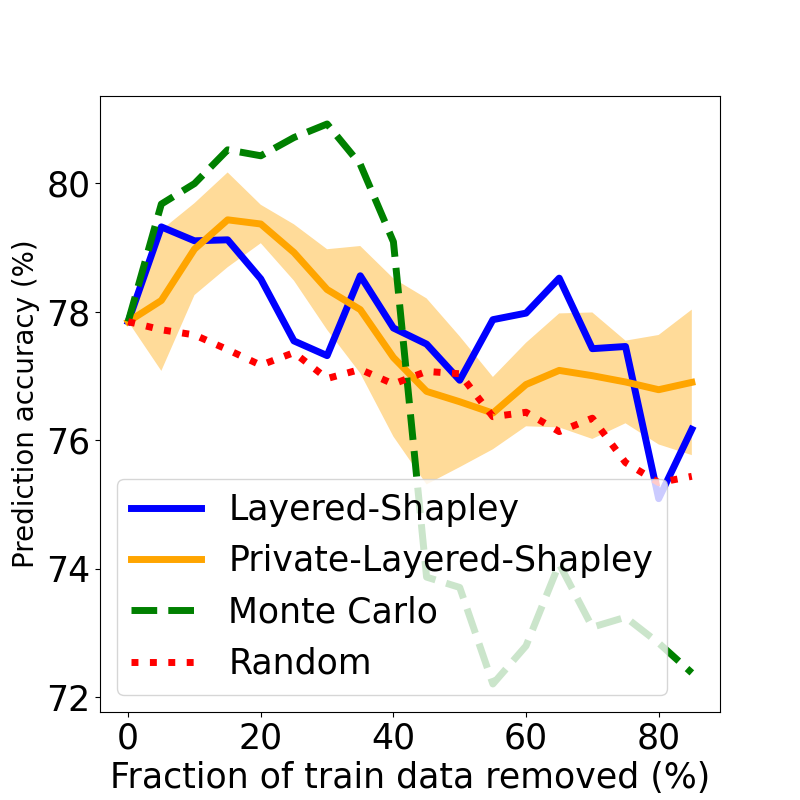}
         \caption{Adult ($\rho=0.61$)}
         \label{fig:y equals x}
     \end{subfigure}
     \hfill
     \begin{subfigure}[b]{0.3\textwidth}
         \centering
         \includegraphics[width=\textwidth]{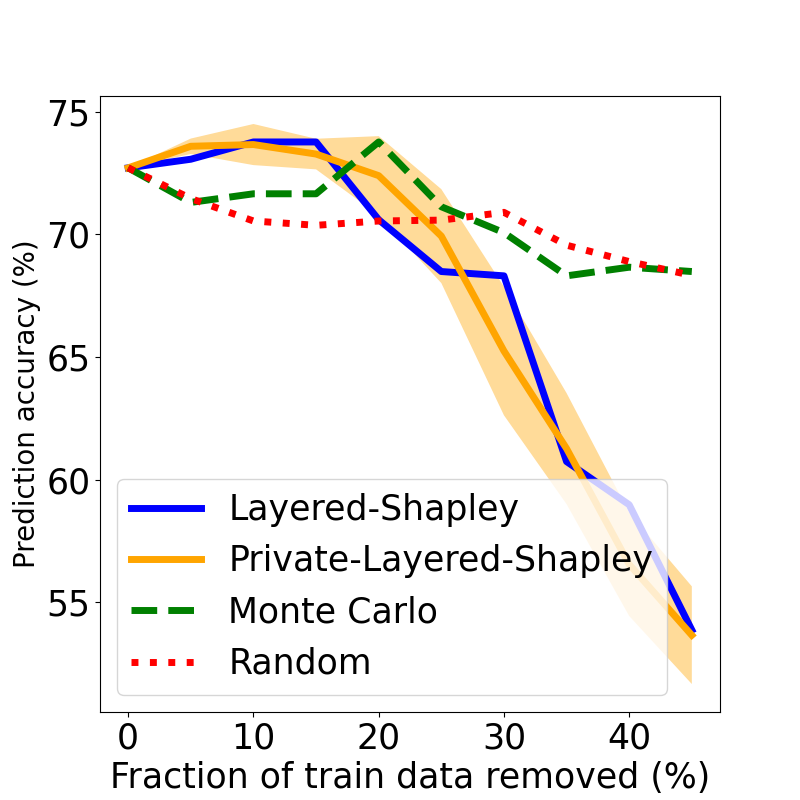}
         \caption{Diabetes ($\rho=0.89$)}
         \label{fig:three sin x}
     \end{subfigure}
     \hfill
     \begin{subfigure}[b]{0.3\textwidth}
         \centering
         \includegraphics[width=0.98\textwidth]{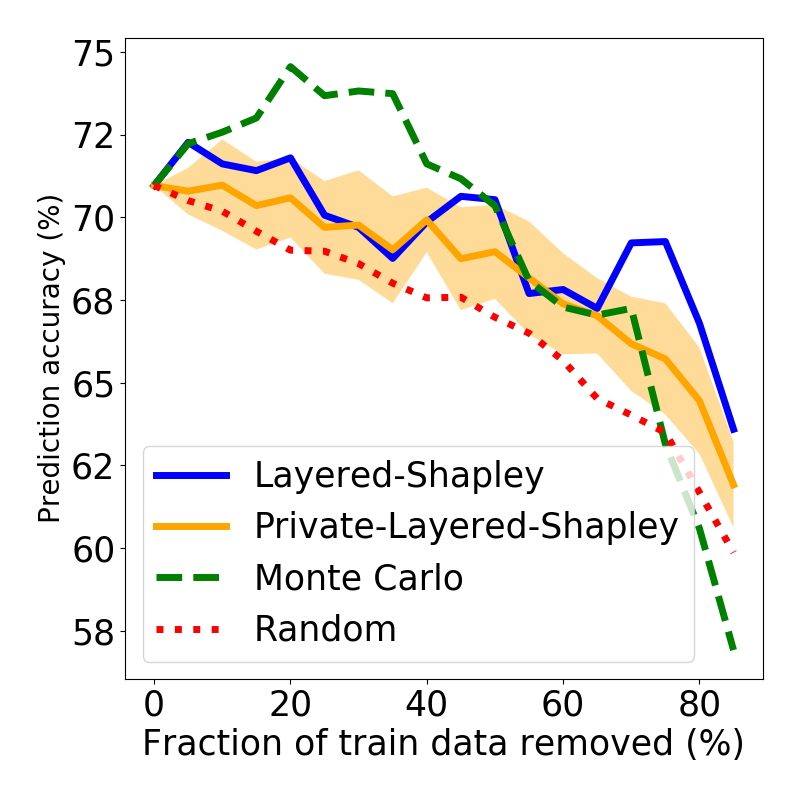}
        \vspace{-3mm}

         \caption{Synthetic Dataset ($\rho=0.52$)}
         \label{fig:five over x}
     \end{subfigure}
        \caption{Shapley value performance demonstrated by the change in accuracy due to removing points ranked by their Shapley value (as opposed to randomly). The top row shows the performance change if the points with the highest Shapley value points are removed first. We expect meaningful Shapley values to result in lines below the red dotted random lines in this case. The bottom row shows the performance change if the smallest Shapley value points are removed first.  We expect meaningful Shapley values to result in lines above the red dotted random lines in this case. The green line shows the performance of the Monte Carlo approximation of Shapley values~\cite{maleki2013bounding}}
        \label{fig:results}
\end{figure}
\textbf{Results:} As shown by Figure~\ref{fig:results}, our private and non-private Shapley value algorithms identify valuable datapoints. In the top row of Figure~\ref{fig:results} removing datapoints with high private or non-private Shapley values results in a faster drop in classifier accuracy in comparison to removing the same number of randomly selected datapoints via the random Shapley value baseline. This implies that the identified points are of higher value to the accuracy of the classifier than randomly selected points. On the bottom row, we see that removing low value points results in a gain in accuracy initially before dropping more slowly in comparison to removing randomly selected points. For the Diabetes dataset, this holds for the bottom 20\% of data only whereas for the other datasets it holds more generally. Overall, this implies that the lowest Shapley value points are those with low value for the classifier in the sense that removing them helps performance.  Together, these results imply that the private Shapley values obtained contain meaningful information about the value of a given datapoint for a classifier, even in the case of differential privacy with $\epsilon=1$. 

We also report the Spearman Rank Correlation $\rho$ between the private and non-private Layered Shapley values. The minimum value for this correlation coefficient is $-1$ implying a perfect negative correlation, $0$ implies no correlation and $1$ implies positive correlation. As the obtained values are significantly above $0$, e.g. an averaged value of $0.61$ or $0.89$ for the Adult and Diabetes datasets respectively, this implies that the private values largely conserve the approximate rank (note that preserving rank precisely will contradict privacy).

%% file: related_work.tex
\section{Related Work}
\label{sec:related}
\textbf{Data Valuation.} Data valuation is a relatively new field and has not been widely addressed until recent years~\cite{jia2019empirical,ghorbani2019data, ghorbani2020distributional, yoon2020data,gonzalez2017fdvt,wei2020efficient}. One main intuition is to value the contribution of different data sources after a model is learned. The valuation can be further used to, for example, make a reasonable payment to each data contributor, which has been discussed and applied in crowdsourcing and federated learning~\cite{jia2019empirical,ghorbani2019data,wei2020efficient}. There are several algorithms for evaluating the data points such as leave-one-out testing~\cite{cook1977detection}, influence function estimation~\cite{koh2017understanding,sharchilev2018finding,datta2015influence,pruthi2020estimating}, and core sets \cite{dasgupta2009sampling}. However, the purposes of these works are mostly for model explanation or stability improvement of models. For tasks such as rewarding the data contributors require additional properties such as fairness or privacy~\cite{li2020federated,li2021survey}. On the other hand, Shapley value as a data valuation method provides axiomatic fairness properties ~\cite{jia2019empirical}.

\textbf{Shapley value in machine learning.} Shapley value has found numerous applications in machine learning~\cite{rozemberczki2022shapley}. Due to the hardness of the computation of exact Shapley values, approximation algorithms for Shapley values are heavily discussed. Maleki, et al.~\cite{maleki2013bounding} provide a general bound on Shapley value with Monte Carlo sampling and show the efficiency of stratified sampling under certain assumptions. In data evaluation applications, Ghorbani, et al. proposed a framework for utilizing Shapley value in a data-sharing system~\cite{ghorbani2019data}. Jia, et al. advanced the work with more detail and several efficient algorithms to approximate the Shapley value under different assumptions~\cite{jia2019empirical}. The distributional Shapley value also has been discussed in~\cite{ghorbani2020distributional,kwon2021efficient}, to address incremental updates to Shapley values, which is difficult under Monte Carlo approximation methods. Their methods calculate the Shapley value over a distribution, without revealing the true Shapley value in the output. 

Several works have explored the use of Shapley values in feature selection and importance. Here, the Shapley values of features quantify how much individual features contribute to the model's performance on a set of data points~\cite{guyon2003introduction,fryer2021shapley, cohen,patel2021game,tripathi2020interpretable,sun2012feature,guha2021cga,williamson2020efficient}. Several different approximation approaches have been proposed for the feature Shapley value. Cohen~\cite{cohen2007feature} assumes the number of interactions between features is significantly smaller than the combinatorial number among all features and derives the Shapley value via coalition sets with only constant sizes. Other works use the variable importance measure (VIM) to quantify the predictive value of each feature, which is called Shapley Population Variable Importance Measure (SPVIM) and can be estimated in $\Theta(d)$ time, where $d$ is the number of features~\cite{williamson2020efficient,covert2021improving}. In general, Shapley value has been widely used as a scoring mechanism in interpretable machine learning~\cite{rozemberczki2022shapley}.

In comparison, our work focuses on the differentially private Shapley values of data points. To the best of our knowledge, this is the first work addressing the differential privacy of data point valuation. Shapley value of data points is a particularly challenging matter since datasets can be large and a single value computation requires many evaluations. Our algorithm operates using smaller samples of data points to obtain privacy-compatible results.

%% file: appendix.tex
\appendix

\section{Proof of Observation 3.1}
First observe that for any $h$, $\abs{L_C(h) - L_{C\cup\{x\}}(h)}\leq \frac{1}{m}$. This is because, if $x$ is correctly classified by $h$ then the difference of the two losses is $\frac{w}{m} - \frac{w}{m+1}$ where $w$ is the number of incorrect classifications. Since $w\leq m$, the difference is at most $1/m$. On the other hand, if $x$ is classified incorrectly, then the difference is $\frac{w}{m} - \frac{w+1}{m+1}\leq 1/m$. 

In the case when with introduction of $x$ the optimal hypothesis does not change, that is $h^\star_{C\cup\{x\}} = h^\star_C$, the observation above applies directly and the difference is at most $\frac{1}{m}$. 

Now, in the event when $h^\star_{C\cup\{x\}} \neq h^\star_C$, we consider two cases. Case 1, is if  $L_C(h^\star_C) \leq L_{C\cup\{x\}}(h^\star_{C\cup\{x\}})$. We know that $L_{C\cup\{x\}}(h^\star_{C\cup\{x\}})\leq L_{C\cup\{x\}}(h^\star_C)$ and that $L_{C\cup\{x\}}(h^\star_C) - L_C(h^\star_C)\leq \frac{1}{m}$. Therefore $ L_{C\cup\{x\}}(h^\star_{C\cup\{x\}}) - L_C(h^\star_C)\leq \frac{1}{m}$. Similarly, in case 2: $L_{C\cup\{x\}}(h^\star_{C\cup\{x\}})\leq L_C(h^\star_C)$, we know that $L_C(h^\star_C)\leq L_C(h^\star_{C\cup\{x\}})$. Which implies that  $L_C(h^\star_C) - L_{C\cup\{x\}}(h^\star_{C\cup\{x\}})\leq \frac{1}{m}$.

\section{Proof of Theorem 3.2}
 In a sample from layer $k$, the probability that a coalition containing $j$ is used is: $\frac{{n-1 \choose k-1}}{{n \choose k}} = \frac{k}{n}$. 

Thus, given $m_k= \frac{c^2}{2\alpha^2 k^2} \ln{\frac{2n}{\beta}}$ expected number of samples at layer $k$ (see discussion above), the probability that $j$ appears in a sampled coalition from stratum $k$ is $P(j|k) \leq \frac{c^2}{2\alpha^2 kn}\ln\frac{2n}{\beta}$. 
Thus, over the all $n$ strata, the probability
\begin{align*}
    P(j) & \leq \sum_{k=1}^n \frac{c^2}{2\alpha^2 kn}\ln\frac{2n}{\beta} \\
         & = \frac{c^2}{2\alpha^2 n}\ln\frac{2n}{\beta} \sum_{k=1}^n \frac{1}{k}\\
         & = \frac{c^2\ln n}{2\alpha^2 n}\ln\frac{2n}{\beta}
\end{align*}

\section{Proof of Theorem 3.3.}
We first show the correctness of $(\alpha,\beta)-$approximation.

\begin{lemma}
The estimate of shapley value $\hat{\varphi}_i$ is an $\alpha, \beta$ approximation of the true shapley value $\varphi_i$. That is, $\Pr(\abs{\hat{\varphi}_i-\varphi_i}\geq \alpha) \leq \beta$. \label{lem:sv-approx}
\end{lemma}
\begin{proof}
Consider the estimate $\hat{\phi}_i^k$ at any layer $k$. By Algorithm 1, $\hat{\phi}_i^k = \frac{1}{m_k}\sum_{C\in S_k}v_i(C)$.

Since  $E[\hat{\phi}_i^k] = \phi_i^k$, using the Chernoff-Hoeffding bound, we have $\Pr(\abs{\hat{\phi}_i^k-\phi_i^k}\geq \alpha)\leq \Pr(\abs{m_k\hat{\phi}_i^k-m_k\phi_i^k}\geq m_k\alpha) \leq 2 \exp \left( -\frac{2\alpha^2m_k^2}{m_k\frac{c^2}{k^2}} \right)$. Substituting the expression for $m_k$, we get that $\Pr(\abs{\hat{\phi}_i^k-\phi_i^k}\geq \alpha)\leq \frac{\beta}{n}$.

By union bound, the probability that $\exists k: \Pr(\abs{\hat{\phi}_i^k - \phi_i^k}\geq\alpha) \leq \sum_{k=1}^n \Pr(\abs{\hat{\phi}_i^k-\phi_i^k}\geq \alpha) \leq \beta$

Observe the event $\sum_{k=1}^n\abs{\hat{\phi}_i^k-\phi_i^k} \geq n\alpha$ requires that $\exists k: \Pr(\abs{\hat{\phi}_i^k - \phi_i^k}\geq\alpha)$. Thus: 
\begin{align*}
    \Pr(\sum_{k=1}^n\abs{\hat{\phi}_i^k-\phi_i^k} \geq n\alpha) \leq \beta \\
    \implies \Pr(\frac{1}{n}\sum_{k=1}^n\abs{\hat{\phi}_i^k-\phi_i^k} \geq \alpha) \leq \beta
\end{align*}

Now, we can rewrite ${\hat{\varphi}_i - \varphi_i}$ as $\frac{1}{n}\sum_{k=1}^n(\hat{\phi}_i^k - \phi_i^k)$. Thus: 
\begin{align*}
\abs{\hat{\varphi}_i - \varphi_i} = \abs{\frac{1}{n}\sum_{k=1}^n(\hat{\phi}_i^k - \phi_i^k)} & \leq \frac{1}{n}\sum_{k=1}^n \abs{(\hat{\phi}_i^k - \phi_i^k)} \\
\implies \Pr(\abs{\hat{\varphi}_i - \varphi_i}\geq \alpha)  \leq \Pr(\frac{1}{n}\sum_{k=1}^n \abs{(\hat{\phi}_i^k - \phi_i^k)}) & \leq \beta. 
\end{align*}
\end{proof}

Now observe that since each of $w_k$ items in layer $k$ is sampled with probability $p_k$, the expected number of samples in layer $k$ is $m_k$. The sample complexity follows from the summation of the sample complexity of the $n$ layers. That is, the sample complexity $m=\sum_{k=1}^n m_k = \sum_{k=1}^n \frac{c^2}{2\alpha^2 k^2}\ln\frac{2n}{\beta} \leq \frac{c^2}{2\alpha^2}\ln\frac{2n}{\beta}\cdot \frac{\pi^2}{6}$. Combining with Lemma~\ref{lem:sv-approx} gives us the theorem.

\section{Proof of Theorem 3.5 }

Denote the set of all sampled coalitions of size $k$ in Algorithm 1 by $\C_k$ and let $m_k=|\C_k|$. Assume that that $v$ is the loss function $\ell(\cdot)$ and that in a given marginal contribution evaluation, the learning algorithm uses a sampled coalition $j \in \C_k$ of size $k$ as its training set. Denote the uniform stability of the learning algorithm using a training set with datasize $k$ by $\gamma_{k}$. Suppose $D'$ is a neighbouring dataset of $D$, differing in at most a single point that is not the point $x_i$ being evaluated. 

When coalition sample $j$ has $k$ data points, let us denote the set of points in it by  $D_{(k,j)} \subseteq D$. Any neighbor of it is written as $D'_{(k,j)} \subseteq D'$ respectively. $D'_{(k,j)}$ and  $D_{(k,j)}$ can differ in at most a single datapoint and so  are also neighbouring datasets. 

Due to the Laplace  Mechanism~\cite{dwork2006calibrating}, noise of scale $\frac{\Delta \hat{\varphi}_i}{\epsilon}$ will suffice to guarantee $\epsilon$-differential privacy, where $\Delta \hat{\varphi}_i$ is the sensitivity of the estimated Shapley Value. By stating $\hat{\varphi}_i$ in terms of the marginal contributions, we can bound the sensitivity as follows,
\begin{align*}
\Delta \hat{\varphi}_i &=& \max_{D, D'} \norm[\Big]{\frac{1}{n} \sum_{k=1}^{n-1} \frac{1}{m_k} \sum_{j \in \C_k}\left(v(D_{(k,j)} \cup x_i)-v(D_{(k,j)} \right) \\
& & - \frac{1}{n}\sum_{k=1}^{n-1}\frac{1}{m_k} \sum_{j \in \C_k}\left(v(D'_{(k,j)} \cup x_i)-v(D'_{(k,j)}) \right)} \\
 &=& \max_{D, D'} \norm[\Big]{\frac{1}{n} \sum_{k=1}^{n-1} \frac{1}{m_k} \sum_{j \in \C_k}\left(v(D_{(k,j)} \cup x_i)-v(D'_{(k,j)} \cup x_i) \right) \\
 && + \frac{1}{n}\sum_{k=1}^{n-1}\frac{1}{m_k} \sum_{j\in \C_k}\left(v(D'_{(k,j)}) -v(D_{(k,j)})\right)} \\
 &\leq& \max_{D, D'} \norm[\Big]{\frac{1}{n} \sum_{k=1}^{n-1} \frac{1}{m_k} \sum_{j \in \C_k}\left(v(D_{(k,j)} \cup x_i)-v(D'_{(k,j)} \cup x_i) \right)} \\
 & & +\max_{D, D'} \norm[\Big]{ \frac{1}{n}\sum_{k=1}^{n-1}\frac{1}{m_k} \sum_{j \in \C_k}\left(v(D'_{(k,j)}) -v(D_{(k,j)})\right)} \\
  &\leq& \frac{1}{n} \sum_{k=1}^{n-1} \frac{1}{m_k} \sum_{j \in \C_k}\gamma_{k+1} + \frac{1}{n}\sum_{k=1}^{n-1}\frac{1}{m_k} \sum_{j \in \C_k}\gamma_{k}  \\
   &\leq& \frac{2}{n} \sum_{k=1}^{n-1} \frac{1}{m_k} \sum_{j \in C_k}\gamma_{k}   \\
   &=& \frac{2}{n} \sum_{k=1}^{n-1} \gamma_{k} \\
   &=&  \frac{L^2\kappa^2}{n \lambda} \sum_{k=1}^{n-1} \frac{1}{k}
\end{align*}
The last three inequalities follow from the fact that the sensitivity is $\propto \frac{1}{k}$ for a coalition of size $k$ and use the uniform stability bound for regularized algorithms given by~\cite{bousquet2002stability}.

Finally, due to Theorem 3.4 and amplification by sampling~\cite{ kellaris2013practical, beimel2014bounds, balle2018privacy}, Laplace noise with scale $ \frac{L^2\kappa^2}{n \lambda \ln(\frac{e^{\epsilon}-1}{p}+1)} \sum_{k=1}^{n-1} \frac{1}{k}$ suffices with $p=\frac{c^2\ln n}{2\alpha^2 n}\ln\frac{2n}{\beta}$.

Note that this sensitivity and noise scale are asymptotically better than what was stated in the theorem statement, since in our algorithm, $m\in O(\polylog(n))$. The main body of the paper will be updated in the camera ready version.

\section{Further Experimental Details}

\textbf{Parameters:} Experiments use the following settings: $\epsilon=1$, $\alpha=0.05$, $\beta=0.05$, $\lambda=1.0$ and $|D|=100$. The data is normalized to the range $(0,1)$ in order to bound $\kappa$.

\textbf{Compute:} The compute requirements of these experiments were low, all experiments were run on laptops using 2.9 GHz Quad-Core Intel Core i7 processors.

%% file: ms.bbl
\begin{thebibliography}{10}

\bibitem{audiffren2013stability}
J.~Audiffren and H.~Kadri.
\newblock Stability of multi-task kernel regression algorithms.
\newblock In {\em Asian Conference on Machine Learning}, pages 1--16. PMLR,
  2013.

\bibitem{bagdasaryan2019differential}
E.~Bagdasaryan, O.~Poursaeed, and V.~Shmatikov.
\newblock Differential privacy has disparate impact on model accuracy.
\newblock {\em Advances in Neural Information Processing Systems}, 32, 2019.

\bibitem{balle2018privacy}
B.~Balle, G.~Barthe, and M.~Gaboardi.
\newblock Privacy amplification by subsampling: tight analyses via couplings
  and divergences.
\newblock In {\em Proceedings of the 32nd International Conference on Neural
  Information Processing Systems}, pages 6280--6290, 2018.

\bibitem{beimel2014bounds}
A.~Beimel, H.~Brenner, S.~P. Kasiviswanathan, and K.~Nissim.
\newblock Bounds on the sample complexity for private learning and private data
  release.
\newblock {\em Machine learning}, 94(3):401--437, 2014.

\bibitem{bousquet2002stability}
O.~Bousquet and A.~Elisseeff.
\newblock Stability and generalization.
\newblock {\em The Journal of Machine Learning Research}, 2:499--526, 2002.

\bibitem{castro2009polynomial}
J.~Castro, D.~G{\'o}mez, and J.~Tejada.
\newblock Polynomial calculation of the shapley value based on sampling.
\newblock {\em Computers \& Operations Research}, 36(5):1726--1730, 2009.

\bibitem{cohen2007feature}
S.~Cohen, G.~Dror, and E.~Ruppin.
\newblock Feature selection via coalitional game theory.
\newblock {\em Neural Computation}, 19(7):1939--1961, 2007.

\bibitem{cohen}
S.~Cohen, E.~Ruppin, and G.~Dror.
\newblock {Feature Selection Based on the Shapley Value}.
\newblock In {\em Proceedings of the 19th International Joint Conference on
  Artificial Intelligence}, page 665–670, 2005.

\bibitem{cook1977detection}
R.~D. Cook.
\newblock Detection of influential observation in linear regression.
\newblock {\em Technometrics}, 1977.

\bibitem{covert2021improving}
I.~Covert and S.-I. Lee.
\newblock {Improving KernelSHAP: Practical Shapley Value Estimation Using
  Linear Regression}.
\newblock In {\em International Conference on Artificial Intelligence and
  Statistics}, pages 3457--3465, 2021.

\bibitem{dasgupta2009sampling}
A.~Dasgupta, P.~Drineas, , et~al.
\newblock Sampling algorithms and coresets for $\backslash$ell\_p regression.
\newblock {\em SIAM Journal on Computing}, pages 2060--2078, 2009.

\bibitem{datta2015influence}
A.~Datta, A.~Datta, et~al.
\newblock {Influence in Classification via Cooperative Game Theory}.
\newblock In {\em Twenty-Fourth International Joint Conference on Artificial
  Intelligence}, 2015.

\bibitem{deng1994complexity}
X.~Deng and C.~H. Papadimitriou.
\newblock On the complexity of cooperative solution concepts.
\newblock {\em Mathematics of operations research}, 19(2):257--266, 1994.

\bibitem{Dwork06differentialprivacy}
C.~Dwork.
\newblock Differential privacy.
\newblock In {\em Automata, Languages and Programming}, pages 1--12. ICALP,
  2006.

\bibitem{dwork2006calibrating}
C.~Dwork, F.~McSherry, K.~Nissim, and A.~Smith.
\newblock Calibrating noise to sensitivity in private data analysis.
\newblock In {\em Theory of cryptography conference}, pages 265--284. Springer,
  2006.

\bibitem{feldman2018generalization}
V.~Feldman and J.~Vondrak.
\newblock Generalization bounds for uniformly stable algorithms.
\newblock {\em Advances in Neural Information Processing Systems}, 31, 2018.

\bibitem{frank2010uci}
A.~Frank and A.~Asuncion.
\newblock Uci machine learning repository [http://archive. ics. uci. edu/ml].
  irvine, ca: University of california.
\newblock {\em School of information and computer science}, 213(2), 2010.

\bibitem{fryer2021shapley}
D.~Fryer, I.~Str{\"u}mke, and H.~Nguyen.
\newblock {Shapley Values for Feature Selection: the Good, the Bad, and the
  Axioms}.
\newblock {\em arXiv:2102.10936}, 2021.

\bibitem{ganev2021robin}
G.~Ganev, B.~Oprisanu, and E.~De~Cristofaro.
\newblock Robin hood and matthew effects--differential privacy has disparate
  impact on synthetic data.
\newblock {\em arXiv preprint arXiv:2109.11429}, 2021.

\bibitem{ghorbani2020distributional}
A.~Ghorbani, M.~Kim, and J.~Zou.
\newblock A distributional framework for data valuation.
\newblock In {\em International Conference on Machine Learning}, pages
  3535--3544. PMLR, 2020.

\bibitem{ghorbani2019data}
A.~Ghorbani and J.~Zou.
\newblock Data shapley: Equitable valuation of data for machine learning.
\newblock In {\em International Conference on Machine Learning}, pages
  2242--2251. PMLR, 2019.

\bibitem{ghorbani2020neuron}
A.~Ghorbani and J.~Zou.
\newblock {Neuron Shapley: Discovering the Responsible Neurons}.
\newblock In {\em Advances in Neural Information Processing Systems}, pages
  5922--5932, 2020.

\bibitem{gonzalez2017fdvt}
J.~Gonz{\'a}lez~Caba{\~n}as, {\'A}.~Cuevas, and R.~Cuevas.
\newblock Fdvt: Data valuation tool for facebook users.
\newblock In {\em Proceedings of the 2017 CHI Conference on Human Factors in
  Computing Systems}, pages 3799--3809, 2017.

\bibitem{guha2021cga}
R.~Guha, A.~H. Khan, et~al.
\newblock Cga: A new feature selection model for visual human action
  recognition.
\newblock {\em Neural Computing and Applications}, 33(10):5267--5286, 2021.

\bibitem{guyon2003introduction}
I.~Guyon and A.~Elisseeff.
\newblock {An Introduction to Variable and Feature Selection}.
\newblock {\em Journal of machine learning research}, 3(Mar):1157--1182, 2003.

\bibitem{hardt2016train}
M.~Hardt, B.~Recht, and Y.~Singer.
\newblock Train faster, generalize better: Stability of stochastic gradient
  descent.
\newblock In {\em International conference on machine learning}, pages
  1225--1234. PMLR, 2016.

\bibitem{jia2019efficient}
R.~Jia, D.~Dao, B.~Wang, F.~A. Hubis, N.~M. Gurel, B.~Li, C.~Zhang, C.~J.
  Spanos, and D.~Song.
\newblock Efficient task-specific data valuation for nearest neighbor
  algorithms.
\newblock {\em arXiv preprint arXiv:1908.08619}, 2019.

\bibitem{jia2019towards}
R.~Jia, D.~Dao, B.~Wang, F.~A. Hubis, N.~Hynes, N.~M. G{\"u}rel, B.~Li,
  C.~Zhang, D.~Song, and C.~J. Spanos.
\newblock Towards efficient data valuation based on the shapley value.
\newblock In {\em The 22nd International Conference on Artificial Intelligence
  and Statistics}, pages 1167--1176. PMLR, 2019.

\bibitem{jia2019empirical}
R.~Jia, X.~Sun, J.~Xu, C.~Zhang, B.~Li, and D.~Song.
\newblock An empirical and comparative analysis of data valuation with scalable
  algorithms.
\newblock 2019.

\bibitem{kellaris2013practical}
G.~Kellaris and S.~Papadopoulos.
\newblock Practical differential privacy via grouping and smoothing.
\newblock {\em Proceedings of the VLDB Endowment}, 6(5):301--312, 2013.

\bibitem{koh2017understanding}
P.~W. Koh and P.~Liang.
\newblock Understanding black-box predictions via influence functions.
\newblock In {\em International Conference on Machine Learning}, pages
  1885--1894, 2017.

\bibitem{kohavi1996scaling}
R.~Kohavi et~al.
\newblock Scaling up the accuracy of naive-bayes classifiers: A decision-tree
  hybrid.
\newblock In {\em Kdd}, volume~96, pages 202--207, 1996.

\bibitem{kumar2020problems}
I.~E. Kumar, S.~Venkatasubramanian, C.~Scheidegger, and S.~Friedler.
\newblock Problems with shapley-value-based explanations as feature importance
  measures.
\newblock In {\em International Conference on Machine Learning}, pages
  5491--5500. PMLR, 2020.

\bibitem{kwon2021efficient}
Y.~Kwon, M.~A. Rivas, and J.~Zou.
\newblock Efficient computation and analysis of distributional shapley values.
\newblock In {\em International Conference on Artificial Intelligence and
  Statistics}, pages 793--801. PMLR, 2021.

\bibitem{li2021survey}
Q.~Li, Z.~Wen, Z.~Wu, S.~Hu, N.~Wang, et~al.
\newblock A survey on federated learning systems: vision, hype and reality for
  data privacy and protection.
\newblock {\em IEEE Transactions on Knowledge and Data Engineering}, 2021.

\bibitem{li2020federated}
T.~Li, A.~K. Sahu, A.~Talwalkar, and V.~Smith.
\newblock Federated learning: Challenges, methods, and future directions.
\newblock {\em IEEE Signal Processing Magazine}, pages 50--60, 2020.

\bibitem{liu2021gtg}
Z.~Liu, Y.~Chen, H.~Yu, Y.~Liu, and L.~Cui.
\newblock {GTG-Shapley: Efficient and Accurate Participant Contribution
  Evaluation in Federated Learning}.
\newblock {\em arXiv:2109.02053}, 2021.

\bibitem{lundberg2017unified}
S.~M. Lundberg and S.-I. Lee.
\newblock A unified approach to interpreting model predictions.
\newblock In {\em Proceedings of the 31st international conference on neural
  information processing systems}, pages 4768--4777, 2017.

\bibitem{maleki2013bounding}
S.~Maleki, L.~Tran-Thanh, G.~Hines, T.~Rahwan, and A.~Rogers.
\newblock {Bounding the Estimation Error of Sampling-based Shapley Value
  Approximation}.
\newblock {\em arXiv:1306.4265}, 2013.

\bibitem{patel2021game}
R.~Patel, M.~Garnelo, I.~Gemp, et~al.
\newblock {Game-Theoretic Vocabulary Selection via the Shapley Value and
  Banzhaf Index}.
\newblock In {\em Proceedings of the 2021 Conference of the North American
  Chapter of the Association for Computational Linguistics}, pages 2789--2798,
  2021.

\bibitem{pedregosa2011scikit}
F.~Pedregosa, G.~Varoquaux, A.~Gramfort, V.~Michel, B.~Thirion, O.~Grisel,
  M.~Blondel, P.~Prettenhofer, R.~Weiss, V.~Dubourg, et~al.
\newblock Scikit-learn: Machine learning in python.
\newblock {\em the Journal of machine Learning research}, 12:2825--2830, 2011.

\bibitem{pruthi2020estimating}
G.~Pruthi, F.~Liu, S.~Kale, and M.~Sundararajan.
\newblock Estimating training data influence by tracing gradient descent.
\newblock {\em Advances in Neural Information Processing Systems},
  33:19920--19930, 2020.

\bibitem{rozemberczki2021shapley}
B.~Rozemberczki and R.~Sarkar.
\newblock {The Shapley Value of Classifiers in Ensemble Games}.
\newblock In {\em Proceedings of the 30th International Conference on
  Information and Knowledge Management}, page 1558–1567, 2021.

\bibitem{rozemberczki2022shapley}
B.~Rozemberczki, L.~Watson, P.~Bayer, H.-T. Yang, O.~Kiss, S.~Nilsson, and
  R.~Sarkar.
\newblock The shapley value in machine learning.
\newblock {\em arXiv preprint arXiv:2202.05594}, 2022.

\bibitem{shalev2014understanding}
S.~Shalev-Shwartz and S.~Ben-David.
\newblock {\em Understanding machine learning: From theory to algorithms}.
\newblock Cambridge university press, 2014.

\bibitem{shapley1953value}
L.~Shapley.
\newblock {A Value for N-Person Games}.
\newblock {\em Contributions to the Theory of Games}, pages 307--317, 1953.

\bibitem{sharchilev2018finding}
B.~Sharchilev, Y.~Ustinovskiy, et~al.
\newblock Finding influential training samples for gradient boosted decision
  trees.
\newblock In {\em International Conference on Machine Learning}, pages
  4577--4585, 2018.

\bibitem{shokri2017membership}
R.~Shokri, M.~Stronati, C.~Song, and V.~Shmatikov.
\newblock Membership inference attacks against machine learning models.
\newblock In {\em 2017 IEEE symposium on security and privacy (SP)}, pages
  3--18. IEEE, 2017.

\bibitem{Sun_2017_ICCV}
C.~Sun, A.~Shrivastava, S.~Singh, and A.~Gupta.
\newblock Revisiting unreasonable effectiveness of data in deep learning era.
\newblock In {\em Proceedings of the IEEE international conference on computer
  vision}, pages 843--852, 2017.

\bibitem{sun2012feature}
X.~Sun, Y.~Liu, J.~Li, et~al.
\newblock {Feature Evaluation and Selection with Cooperative Game Theory}.
\newblock {\em Pattern recognition}, 45(8):2992--3002, 2012.

\bibitem{tripathi2020interpretable}
S.~Tripathi, N.~Hemachandra, and P.~Trivedi.
\newblock {Interpretable Feature Subset Selection: A Shapley Value Based
  Approach}.
\newblock In {\em IEEE International Conference on Big Data}, pages 5463--5472,
  2020.

\bibitem{wei2020efficient}
S.~Wei, Y.~Tong, Z.~Zhou, and T.~Song.
\newblock Efficient and fair data valuation for horizontal federated learning.
\newblock In {\em Federated Learning}, pages 139--152. Springer, 2020.

\bibitem{williamson2020efficient}
B.~Williamson and J.~Feng.
\newblock {Efficient Nonparametric Statistical Inference on Population Feature
  Importance Using Shapley Values}.
\newblock In {\em International Conference on Machine Learning}, pages
  10282--10291, 2020.

\bibitem{yoon2020data}
J.~Yoon, S.~Arik, and T.~Pfister.
\newblock Data valuation using reinforcement learning.
\newblock In {\em International Conference on Machine Learning}, pages
  10842--10851. PMLR, 2020.

\end{thebibliography}
